\documentclass[11pt]{article}
\usepackage{fullpage}

\usepackage{verbatim}

\usepackage[round]{natbib}

    \usepackage[usenames]{color}
    \definecolor{plum}  {rgb}{.4,0,.4}
    \definecolor{BrickRed} {rgb}{0.6,0,0}
	\usepackage[plainpages=false,pdfpagelabels,colorlinks=true,linkcolor=BrickRed,citecolor=plum]{hyperref}

\usepackage[utf8]{inputenc} % allow utf-8 input
\usepackage[T1]{fontenc}    % use 8-bit T1 fonts
\usepackage{hyperref}       % hyperlinks
\usepackage{url}            % simple URL typesetting
\usepackage{booktabs}       % professional-quality tables
\usepackage{amsfonts}       % blackboard math symbols
\usepackage{nicefrac}       % compact symbols for 1/2, etc.
\usepackage{microtype}      % microtypography
\usepackage{algorithm2e}
\usepackage{graphicx}
\usepackage{xcolor}

\usepackage{amsmath,amssymb,amsthm}
\usepackage{commath}

\usepackage[mathcal]{euscript}

    \def\ddefloop#1{\ifx\ddefloop#1\else\ddef{#1}\expandafter\ddefloop\fi}

    \def\ddef#1{\expandafter\def\csname c#1\endcsname{\ensuremath{\mathcal{#1}}}}
    \ddefloop ABCDEFGHIJKLMNOPQRSTUVWXYZ\ddefloop

    \def\ddef#1{\expandafter\def\csname s#1\endcsname{\ensuremath{\mathsf{#1}}}}
    \ddefloop ABCDEFGHIJKLMNOPQRSTUVWXYZ\ddefloop
	
    \def\ddef#1{\expandafter\def\csname bb#1\endcsname{\ensuremath{\mathbb{#1}}}}
    \ddefloop ABCDEFGHIJKLMNOPQRSTUVWXYZ\ddefloop
	
    \def\ddef#1{\expandafter\def\csname bd#1\endcsname{\ensuremath{\boldsymbol{#1}}}}
    \ddefloop ABCDEFGHIJKLMNOPQRSTUVWXYZ\ddefloop
	
    \def\ddef#1{\expandafter\def\csname bd#1\endcsname{\ensuremath{\boldsymbol{#1}}}}
    \ddefloop abcdefghijklmnopqrstuvwxyz\ddefloop
	
    \def\E{\mathbf{E}}

    \def\Reals{\mathbb{R}}

    \def\deq{:=}

    \def\wh#1{\widehat{#1}}
    \def\bd#1{\boldsymbol{#1}}

    \def\bP{\bd{P}}

    \def\1{{\mathbf 1}}

	\def\bdmu{\boldsymbol{\mu}}
	\def\bdnu{\boldsymbol{\nu}}

    \newtheorem{theorem}{Theorem}[section]

\begin{document}

\title{Neural Stochastic Differential Equations:\\
Deep Latent Gaussian Models in the Diffusion Limit}

\author{Belinda Tzen\thanks{University of Illinois, e-mail: btzen2@illinois.edu.} \and Maxim Raginsky\thanks{University of Illinois, e-mail: maxim@illinois.edu.}}

\date{}

\maketitle

\begin{abstract}
In deep latent Gaussian models, the latent variable is generated by a time-inhomogeneous Markov chain, where at each time step we pass the current state through a parametric nonlinear map, such as a feedforward neural net, and add a small independent Gaussian perturbation. This work considers the diffusion limit of such models, where the number of layers tends to infinity, while the step size and the noise variance tend to zero. The limiting latent object is an It\^o diffusion process that solves a stochastic differential equation (SDE) whose drift and diffusion coefficient are implemented by neural nets. We develop a variational inference framework for these \textit{neural SDEs} via stochastic automatic differentiation in Wiener space, where the variational approximations to the posterior are obtained by Girsanov (mean-shift) transformation of the standard Wiener process and the computation of gradients is based on the theory of stochastic flows. This permits the use of black-box SDE solvers and automatic differentiation for end-to-end inference. Experimental results with synthetic data are provided.
\end{abstract}

\section{Introduction}

Ordinary differential equations (ODEs) and other types of continuous-time flows have always served as convenient abstractions for various deterministic iterative models and algorithms. Recently, however, several authors have started exploring the intriguing possibility of using ODEs for constructing and training very deep neural nets by considering the limiting case of composing a large number of infinitesimal nonlinear transformations \citep{haber2017diffeq,chen18neuralODE,li2018maximumDP}. In particular, \citet{chen18neuralODE} have introduced the framework of \textit{neural ODEs}, in which the overall nonlinear transformation is represented by an ODE, and a black-box ODE solver is used as a computational primitive during end-to-end training. 

These ideas naturally carry over to the domain of probabilistic modeling. Indeed, since deep probabilistic generative models can be viewed as time-inhomogeneous Markov chains, we can consider the limit of infinitely many layers as a continuous-time Markov process. Since the marginal distributions of such a process evolve through a deterministic continuous-time flow, one can use ODE techniques in this context as well \citep{tabak10dualascent,chen18flows}. However, an alternative possibility is to focus on the stochastic evolution of the sample paths of the limiting process, rather than on the deterministic evolution in the space of measures. This perspective is particularly useful when the generation of sample paths of the underlying process is more tractable than the computation of  process distributions.

In this paper, we develop these ideas in the context of Deep Latent Gaussian Models (DLGMs), a flexible family of generative models introduced by \citet{rezende2014stochbackprop}. In these models, the latent variable is generated by a time-inhomogeneous Markov chain, where at each time step we pass the current state through a deterministic nonlinear map, such as a feedforward neural net, and add a small independent Gaussian perturbation. The observed variable is then drawn conditionally on the state of the chain after a large but finite number of steps. The iterative structure of DLGMs, together with the use of differentiable layer-to-layer transformations, is the basis of \textit{stochastic backpropagation} \citep{kingma14VAE,ranganath14blackboxvi,rezende2014stochbackprop}, an efficient and scalable procedure for performing variational inference with approximate posteriors of the mean-field type. A key feature here is that all the randomness in the latent space is generated by sampling a large but finite number of independent standard Gaussian random vectors, and all other transformations are obtained by suitable differentiable reparametrizations.

If one considers the limiting regime of DLGMs, where the number of layers tends to infinity while the step size and the noise variance in layer-to-layer transformations both tend to zero, the resulting latent random object is a \textit{diffusion process} of the It\^o type \citep{bichteler2002stochastic,protter2005SDE}, whose drift and diffusion coefficients are implemented by neural nets. We will refer to these models as \textit{neural SDEs}, in analogy to the deterministic neural ODEs of \citet{chen18neuralODE}. Generative models of this type have been considered in earlier work, first by \cite{movellan2002diffusions} as a noisy continuous-time counterpart of recurrent neural nets, and, more recently, by \cite{archambeau2007diffusions}, \cite{hashimoto16}, \cite{ha2018pathAE}, and \cite{ryder2018SDE_VI}. On the theoretical side, \citet{tzen19SDE} have investigated the expressive power of diffusion-based generative models and showed that they can be used to obtain approximate samples from any distribution whose Radon--Nikodym derivative w.r.t.\ the standard Gaussian measure can be efficiently represented by a neural net. In this paper, we leverage this expressive power and develop a framework for variational inference in neural SDEs:
\begin{itemize}
	\item We show that all the latent randomness can be generated by sampling from the standard multidimensional Wiener process, in analogy to the use of independent standard Gaussian random vectors in DLGMs. Thus, the natural latent space for neural SDEs is the \textit{Wiener space}  of continuous vector-valued functions on $[0,1]$ equipped with the Wiener measure (the probability law of the standard Wiener process).
	\item We derive a variational bound on the marginal log-likelihood for the observed variable using the Gibbs variational principle on the path space \citep{boue1998variational}. Moreover, by Girsanov's theorem, any variational approximation to the posterior is related to the primitive Wiener process by a mean shift. Thus, the natural neural SDE counterpart of a mean-field approximate posterior is obtained by adding an observation-dependent neural net drift to the standard Wiener process.
	\item Finally, we show how variational inference can be carried out via automatic differentiation (AD) in Wiener space. One of the salient features of the neural ODE framework of \citet{chen18neuralODE} is that one can backpropagate gradients efficiently through any black-box ODE solver using the so-called \textit{method of adjoints} (see, e.g., \citet{kokotovic67adjoints} and references therein). While there exists a counterpart of the method of adjoints for SDEs \citep[Chap.~3]{yongzhou_HJB}, it cannot be used to backpropagate gradients through a black-box SDE solver, as we explain in Section~\ref{ssec:SDEgrad}. Instead, one has to either derive custom backpropagation rules for each specific solver or use the less time-efficient forward-mode AD with a black-box SDE solver. For the latter, we use the theory of stochastic flows \citep{kunita84flows} in order to differentiate the solutions of It\^o SDEs with respect to parameters of the drift and the diffusion coefficient. These pathwise derivatives are also solutions of It\^o SDEs, whose drift and diffusion coefficient can be obtained from those of the original SDE using the ordinary chain rule of multivariable calculus. Thus, the overall process can be implemented using AD and a black-box SDE solver.
\end{itemize}

\subsection{Related work}

Extending the neural ODE framework of \citet{chen18neuralODE} to the setting of SDEs is a rather natural step that has been taken by several authors. In particular,  the use of SDEs to enhance the expressive power of continuous-time neural nets was proposed in a concurrent work of \citet{peluchetti2019neuralSDE}. Neural SDEs driven by stochastic processes with jumps were introduced by \citet{jia2019jump} as a generative framework for hybrid dynamical systems with both continuous and discrete behavior. \citet{hegde2019GP} considered generative models built from SDEs whose drift and diffusion coefficients are samples from a Gaussian process. \cite{liu2019neuralSDE} and \citet{wang2019resnet}  have proposed using SDEs (and suitable discretizations) as a noise injection mechanism to stabilize neural nets against adversarial or stochastic input perturbations.

\section{Background: variational inference in Deep Latent Gaussian Models}
\label{sec:DLGM}

In Deep Latent Gaussian Models (DLGMs) \citep{rezende2014stochbackprop}, the latent variables $X_0,\ldots,X_k$ and the observed variable $Y$ are generated recursively:
\begin{subequations}\label{eq:DLGM}
\begin{align}
	X_0 &= Z_0 \label{eq:x0}\\
	X_i &= X_{i-1} + b_i(X_{i-1}) + \sigma_i Z_i, \qquad i = 1, \ldots, k \label{eq:xk}\\
	Y & \sim p(\cdot|X_k),
\end{align}
\end{subequations}
where $Z_0,\ldots,Z_k \stackrel{{\rm i.i.d.}}{\sim} \cN(0,I_d)$ are i.i.d.\ standard Gaussian vectors in $\Reals^d$, $b_1,\ldots,b_k : \Reals^d \to \Reals^d$ are some parametric nonlinear transformations,  $\sigma_1,\ldots,\sigma_k \in \Reals^{d \times d}$ is a sequence of matrices, and $p(\cdot|\cdot)$ is the observation likelihood. Letting $\theta$ denote the parameters of $b_1,\ldots,b_k$ and the matrices $\sigma_1,\ldots,\sigma_k$, we can capture the underlying generative process by the joint probability density of $Y$ and the `primitive' random variables $Z_0,\ldots,Z_k$:
\begin{align}\label{eq:DLGM_density}
	p_\theta(y,z_0,\ldots,z_k) = p(y|f_\theta(z_0,\ldots,z_k) \phi_d(z_0) \ldots \phi_d(z_k),
\end{align}
where $f_\theta : \Reals^{d} \times \ldots \times \Reals^d \to \Reals^d$ is the overall deterministic transformation $(Z_0,\ldots,Z_k) \mapsto X_k$ specified by \eqref{eq:x0}-\eqref{eq:xk}, and $\phi_d(z) = (2\pi)^{-d/2}\exp(-\frac{1}{2}\|z\|^2)$ is the standard Gaussian density in $\Reals^d$. 

The main object of inference is the marginal likelihood $p_\theta(y)$, obtained by integrating out the latent variables $z_0,\ldots,z_k$ in \eqref{eq:DLGM_density}. This integration is typically intractable, so instead one works with the so-called \textit{evidence lower bound}. Typically, this bound is derived using Jensen's inequality (see, e.g., \citet{blei2017VI}); for our purposes, though, it will be convenient to derive it from the well-known \textit{Gibbs variational principle} \citep{dupuis1997largedev}: For any Borel probability measure $\mu$ on $\Omega \deq (\Reals^d)^{k+1}$ and any measurable real-valued function $F : \Omega \to \Reals$,
\begin{align}\label{eq:Gibbs}
	-\log \E_\mu[e^{-F(Z_0,\ldots,Z_k)}] = \inf_{\nu \in \cP(\Omega)} \left\{ D(\nu\|\mu) + \E_\nu[F(Z_0,\ldots,Z_k)]\right\},
\end{align}
where $D(\cdot\|\cdot)$ is the Kullback--Leibler divergence and the infimum is over all Borel probability measures $\nu$ on $\Omega$. If we let $\mu$ be the marginal distribution of $Z_0,\ldots,Z_k$ in \eqref{eq:DLGM_density} and apply \eqref{eq:Gibbs} to the function $F_\theta(z_0,\ldots,z_k) \deq - \log p(y|f_\theta(z_0,\ldots,z_k))$, we obtain the well-known variational formula
\begin{align}
	-\log p_\theta(y) &= - \log \int p(y|f_\theta(z_0,\ldots,z_k))\phi_d(z_0)\ldots\phi_d(z_k) \dif z_0 \ldots \dif z_k \nonumber\\
	&= \inf_{\nu \in \cP(\Omega)} \Big\{ D(\nu \| \mu) - \int_\Omega \log p(y|f_\theta(z_0,\ldots,z_k))\nu(\dif z_0,\ldots,\dif z_k)\Big\}. \label{eq:DLGM_Gibbs}
\end{align}
The infimum in \eqref{eq:DLGM_Gibbs} is attained by the posterior density $p_\theta(z_0,\ldots,z_k|y)$  whose computation is also generally intractable, so one typically picks a suitable family $\nu_\beta(\dif z_0,\ldots,\dif z_k|y) = q_\beta(z_0,\ldots,z_k|y) \dif z_0 \ldots \dif z_k$ of approximate posteriors to obtain the variational upper bound
\begin{align}\label{eq:DLGM_VB}
	&-\log p_\theta(y) \le \inf_\beta \sF_{\theta,\beta}(y),
\end{align}
where
\begin{align*}
	&\sF_{\theta,\beta}(y) \deq D(\nu_\beta(\cdot|y) \| \mu) - \int_\Omega \log p(y|f_\theta(z_0,\ldots,z_k)) \nu_\beta(\dif z_0,\ldots,\dif z_k|y)
\end{align*}
is the \textit{variational free energy}. The choice of $q_\beta(\cdot|y)$ is driven by considerations of computational tractability vs.\ representational richness. A widely used family of approximate posteriors is given by the \textit{mean-field approximation}: $q_\beta(\cdot|y)$ is the product of $k+1$ Gaussian densities whose means $\tilde{b}_0(y),\ldots,\tilde{b}_k(y)$ and covariance matrices $C_0(y),\ldots,C_k(y)$ are also chosen from some parametric class of nonlinearities, and $\beta$ is then the collection of all the parameters of these transformations. The resulting inference procedure, known as \textit{stochastic backpropagation} \citep{kingma14VAE,rezende2014stochbackprop}, involves estimating the gradients of the variational free energy $\sF_{\beta,\theta}(y)$ with respect to $\theta$ and $\beta$. This procedure hinges on two key steps:

\paragraph{Reparametrization:} The joint law of $\tilde{b}_i(y) + C^{1/2}_i(y)Z_i$, $i \in \{0,\ldots,k\}$ with $Z_i \stackrel{{\rm i.i.d.}}{\sim} \cN(0,I_d)$ is equal to the mean-field posterior $\nu_\beta(\cdot|y)$. Thus, we can express the variational free energy as
\begin{align}
	\sF_{\theta,\beta}(y) = \sum^k_{i=0} D(\cN(\tilde{b}_i(y),C_0(y))\|\cN(0,I_d))   + \E\left[F_\theta\big(\tilde{b}_0(y)+C_0(y)^{1/2}Z_0,\ldots,\tilde{b}_k(y)+C_k(y)^{1/2}Z_0\big)\right]. \label{eq:FE}
\end{align}
\paragraph{Backpropagation with Monte Carlo:} Since the expectation in \eqref{eq:FE} is w.r.t.\ to a collection of i.i.d.\ standard Gaussian vectors, it follows that the gradients can be computed by interchanging differentiation and expectation and using reverse-mode automatic differentiation or backpropagation \citep{baydin2018autodiff}. Unbiased estimates of $\nabla_\bullet \sF_{\theta,\beta}(y)$, $\bullet \in \{\theta,\beta\}$, can then be obtained by Monte Carlo sampling.

\section{Neural Stochastic Differential Equations as DLGMs in the diffusion limit}

In this work, we consider the continuous-time limit of \eqref{eq:DLGM}, in analogy to the neural ODE framework of \citet{chen18neuralODE} (which corresponds to the deterministic case $\sigma_i \equiv 0$). In this limit, the latent object becomes a $d$-dimensional \textit{diffusion process} $X = \{X_t\}_{t \in [0,1]}$ given by the solution of the It\^o stochastic differential equation (SDE)
\begin{align}\label{eq:neural_SDE}
	\dif X_t = b(X_t,t)\dif t + \sigma(X_t,t)\dif W_t,  \,\,\ t \in [0,1]
\end{align}
where $W$ is the standard $d$-dimensional Wiener process or Brownian motion (see, e.g., \citet{bichteler2002stochastic} or \citet{protter2005SDE} for background on diffusion processes and SDEs). The observed variable $Y$ is generated conditionally on $X_1$: $Y \sim p(\cdot|X_1)$. We focus on the case when both the drift $b : \Reals^d \times [0,1] \to \Reals^d$ and the diffusion coefficient $\sigma : \Reals^d \times [0,1] \to \Reals^{d \times d}$ are implemented by feedforward neural nets, and will thus use the term \textit{neural SDE} to refer to \eqref{eq:neural_SDE}.

The special case of \eqref{eq:neural_SDE} with $X_0 = 0$ and $\sigma \equiv I_d$ was studied by \citet{tzen19SDE}, who showed that generative models of this kind are sufficiently expressive. In particular, they showed that one can use a neural net drift $b(\cdot,\cdot)$ and $\sigma = I_n$ in \eqref{eq:neural_SDE} to obtain approximate samples from any target distribution  for $X_1$ whose Radon--Nikodym derivative $f$ w.r.t.\ the standard Gaussian measure on $\Reals^d$ can be represented efficiently using neural nets. Moreover, only a polynomial overhead is incurred in constructing $b$ compared to the neural net representing $f$. 

The main idea of the construction of \citet{tzen19SDE} can be informally described as follows. Consider a target density of the form $q(x) = f(x)\phi_d(x)$ for a sufficiently smooth $f : \Reals^d \to \Reals_+$ and the It\^o SDE
\begin{align}\label{eq:Follmer}
	\dif X_t = \frac{\partial}{\partial x} \log Q_{1-t}f(X_t) \dif t + \dif W_t, \,\, X_0 = 0;\, t \in [0,1]
\end{align}
where $Q_tf(x) \deq \E_{Z \sim \phi_d}[f(x+\sqrt{t}Z)]$. Then $X_1 \sim q$, i.e., one can use \eqref{eq:Follmer} to obtain an exact sample from $q$, and this construction is information-theoretically optimal (see, e.g., \citet{daipra1991reciprocal}, \citet{lehec2013entropy},  \cite{eldan2018diffusion}, or \citet{tzen19SDE}). The drift term in \eqref{eq:Follmer} is known as the \textit{F\"ollmer drift} \citep{follmer1985reversal}. Replacing $Q_{1-t}f(x)$ by a Monte Carlo estimate and $f(\cdot)$ by a suitable neural net approximation $\wh{f}(\cdot;\theta)$, we can approximate the F\"ollmer drift by functions of the form
\begin{align*}%\label{eq:apx_Follmer}
	\wh{b}(x,t;\theta) = \frac{\partial}{\partial x} \log \left\{ \frac{1}{N}\sum^N_{n=1} \wh{f}(x + \sqrt{1-t}z_n; \theta) \right\} = \frac{\sum^N_{n=1} \frac{\partial}{\partial x}\wh{f}(x+\sqrt{1-t}z_n;\theta)}{\sum^N_{n=1} \wh{f}(x+\sqrt{1-t}z_n;\theta)}.
\end{align*}
This has the following implications (see \citet{tzen19SDE} for a detailed analysis):
\begin{itemize}
	\item the complexity of representing the F\"ollmer drift by a neural net is comparable to the complexity of representing the Radon--Nikodym derivative $f = \frac{q}{\phi_d}$ by a neural net;
	\item the neural net approximation to the F\"ollmer drift takes both the space variable $x$ and the time variable $t$ as inputs, and its weight parameters $\theta$ do not explicitly depend on time.
\end{itemize}
In some cases, this can be confirmed by direct computation. As an example, consider a stochastic deep linear neural net \citep{hardt17identity} in the diffusion limit:
\begin{align}\label{eq:deep_linear}
	\dif X_t = A_t X_t \dif t + C_t \dif W_t, \qquad X_0 = x_0;\, t \in [0,1].
\end{align}
In this representation, the net is parametrized by the matrix-valued paths $\{A_t\}_{t \in[0,1]}$ and $\{C_t\}_{t \in [0,1]}$, and optimizing over the model parameters is difficult even in the deterministic case. On the other hand, the process $X$ in \eqref{eq:deep_linear} is Gaussian (in fact, all Gaussian diffusion processes are of this form), and the probability law of $X_1$ can be computed in closed form \citep[Chap.~V, Sec.~9]{fleming1975control}: $X_1 \sim \cN(m,\Sigma)$ with
\begin{align*}
	m = \Phi_{0,1}x_0 \qquad \text{and} \qquad \Sigma = \int^1_0 \Phi_{t,1} C_t C^\top_t \Phi_{t,1}^\top \dif t,
\end{align*}
where $\Phi_{s,t}$ (for $s < t$) is the \textit{fundamental matrix} that solves the ODE
\begin{align*}
	\frac{\dif}{\dif t} \Phi_{s,t} = A_t \Phi_{s,t}, \qquad \Phi_s = I_d;\, t > s
\end{align*}
The F\"ollmer drift provides a more parsimonious representation that does not involve time-varying network parameters. Indeed, since $q$ is the $d$-dimensional Gaussian density with mean $m$ and covariance matrix $\Sigma$, we have
\begin{align*}
	f(x) = c\,\exp\left\{ - \frac{1}{2}\Big((x-m)^\top \Sigma^{-1} (x-m) + x^\top x\Big)\right\},
\end{align*}
where $c$ is a normalization constant. If $\det \Sigma \neq 0$, a straightforward but tedious computation yields
\begin{align*}
	Q_tf(x) &= c_t \exp\Big\{ \frac{1}{2}\Big(t v^\top \Sigma_tv - (x-m)^\top \Sigma^{-1} (x-m) + x^\top x\Big)\Big\},
\end{align*}
where $c_t > 0$ is a constant that does not depend on $x$, $v = (\Sigma^{-1}-I_d)x - \Sigma^{-1}m$, and $\Sigma_t = ((1-t)I_d + t\Sigma^{-1})^{-1}$. Consequently, the F\"ollmer drift is given by
\begin{align*}
	\frac{\partial}{\partial x} \log Q_{1-t}f(x) &= [(1-t)(\Sigma^{-1}-I_d)\Sigma_{1-t}(\Sigma^{-1}-I_d)-(\Sigma^{-1}-I_d)]x \nonumber\\
	& \qquad - [(1-t)(\Sigma^{-1}-I_d)\Sigma_{1-t}\Sigma^{-1} - \Sigma^{-1}]m,
\end{align*}
which is an affine function of $x$ with time-invariant parameters $\Sigma^{-1}$ and $m$.

\section{Variational inference with neural SDEs}

Our objective here is to develop a variational inference framework for neural SDEs that would leverage their expressiveness and the availability of adaptive black-box solvers for SDEs \citep{ilie2015sdesolve}. We start by showing that all the building blocks of DLGMs described in Section~\ref{sec:DLGM} have their natural counterparts in the context of neural SDEs.

\paragraph{Brownian motion as the latent object:} In the case of DLGMs, it was expedient to push all the randomness in the latent space $\Omega = (\Reals^d)^{k+1}$ into the i.i.d.\ standard Gaussians $Z_0,\ldots,Z_k$. An analogous procedure can be carried out for neural SDEs as well, except now the latent space is $\bbW = C([0,1];\Reals^d)$, the space of continuous paths $w : [0,1] \to \Reals^d$, and the primitive random object is the Wiener process $W = \{W_t\}_{t \in [0,1]}$. The continuous-time analogue of the independence structure of the $Z_i$'s is the independent Gaussian increment property of $W$: for any $0 \le s < t \le 1$, the increment $W_t-W_s \sim \cN(0,(t-s)I_d)$ is independent of $\{W_r : 0 \le r \le s\}$. Moreover, there is a unique probability law $\bdmu$ on $\bbW$ (the \textit{Wiener measure}), such that, under $\bdu$, $W_0 = 0$ almost surely and, for any $0 < t_1 < t_2 < \ldots < t_m \le 1$,  $W_{t_i}-W_{t_{i-1}}$ with $t_0 = 0$ are independent centered Gaussian random vectors with covariance matrices $(t_i - t_{i-1})I_d$.

Let us explicitly parametrize the drift and the diffusion in \eqref{eq:neural_SDE} as $b(x,t;\theta)$ and $\sigma(x,t;\theta)$, respectively. If, for each $\theta$, $b$ and $\sigma$ in \eqref{eq:neural_SDE} are Lipschitz-continuous in $x$ uniformly in $t \in [0,1]$, then there exists a mapping $f_\theta : \bbW \to \bbW$, such that $X = f_\theta(W)$, and this mapping is \textit{progressively measurable} \citep[Sec.~5.2]{bichteler2002stochastic}: If we denote by $[f_\theta(W)]_t$ the value of $f_\theta(W)$ at $t$, then
\begin{align*}
	[f_\theta(W)]_t &= \int^t_0 b([f_\theta(W)]_s,s; \theta)\dif s  + \int^t_0 \sigma([f_\theta(W)]_s,s; \theta)\dif W_s,
\end{align*}
that is, for each $t$, the path $\{[f_\theta(W)]_s\}_{s \in [0,t]}$ depends only on $\{W_s\}_{s \in [0,t]}$. With these ingredients in place, we have the following path-space analogue of \eqref{eq:DLGM_density}:
\begin{align}
	\bP_\theta(\dif y, \dif w) =  p(y|[f_\theta(w)]_1)\bdmu(\dif w) \dif y.
\end{align}
As before, the quantity of interest is the marginal density $p_\theta(y) \deq \int_\bbW p(y|[f_\theta(w)]_1)\bdmu(\dif w)$.

\paragraph{The variational representation and Girsanov reparametrization:} The Gibbs variational principle also holds for probability measures and measurable real-valued functions on the path space $\bbW$ \citep{boue1998variational}, so we obtain the following variational formula:
\begin{align}\label{eq:diffusion_VB}
	-\log p_\theta(y) 
	&= \inf_{\bdnu \in \cP(\bbW)} \left\{ D(\bdnu\|\bdmu) - \int_\bbW \log p(y|[f_\theta(w)]_1) \bdnu(\dif w) \right\}.
\end{align}
This formula looks forbidding, as it involves integration with respect to probability measures $\bdnu$ on path space $\bbW$. However, a significant simplification comes about from the fundamental result known as \textit{Girsanov's theorem} \citep[Prop.~3.9.13]{bichteler2002stochastic}: any probability measure $\bdnu$ on $\bbW$ which is absolutely continuous w.r.t.\ the Wiener measure $\bdmu$ corresponds to the addition of a drift term to the basic Wiener process $W$. That is, there is a one-to-one correspondence between $\bdnu \in \cP(\bbW)$ with $D(\bdnu\|\bdmu) < \infty$ and $\Reals^d$-valued random processes $u = \{u_t\}_{t \in [0,1]}$, such that each $u_t$ is measurable w.r.t.\ $\{W_s : 0 \le s \le t\}$ and $\E_{\bdmu} \left[ \frac{1}{2}\int^1_0 \|u_t\|^2 \dif t\right] < \infty$. Specifically, if we consider the It\^o process
\begin{align}\label{eq:drift_added}
	Z_t = \int^t_0 u_s \dif s + W_t, \qquad t \in [0,1],
\end{align}
then $Z=\{Z_t\}_{t \in [0,1]} \sim \bdnu$ with
$$
D(\bdnu\|\bdmu) = \E_{\bdmu}\left[\frac{1}{2}\int^1_0\|u_t\|^2\dif t\right].
$$
Conversely, any such $\bdnu$ can be realized in this fashion. This leads to the  \textit{Girsanov reparametrization} of the variational formula \eqref{eq:diffusion_VB}:
\begin{align*}
	-\log p_\theta(y)
&= \inf_u \E_{\bdmu} \left\{ \frac{1}{2}\int^1_0 \|u_t\|^2 \dif t + F_\theta\left( W + \int^\bullet_0 u_s \dif s\right)\right\},
\end{align*}
where $W +  \int^\bullet_0 u_s \dif s$ is shorthand for the process $\{W_t + \int^t_0 u_s \dif s\}_{t \in[0,1]}$, and $F_\theta(w) \deq -\log p(y|[f_\theta(w)]_1)$.

\paragraph{Mean-field approximation:} The next order of business is to develop a path-space analogue of the mean-field approximation. This is rather simple: we consider \textit{deterministic} drifts of the form $u_t = \tilde{b}(y,t; \beta)$, $t \in [0,1]$, where $\tilde{b}(t,y; \beta)$ is a neural net with weight parameters $\beta$, such that $\int^1_0 \|\tilde{b}(y,t; \beta)\|^2 \dif t < \infty$. The resulting process \eqref{eq:drift_added} is Gaussian and has independent increments with
$$
Z_t - Z_s \sim \cN\left(\int^t_s \tilde{b}(y,s';\beta)\dif s', (t-s)I_d\right),
$$
and we have the mean-field variational bound
\begin{align*}
	-\log p_\theta(y) \le \inf_\beta \Bigg\{ \frac{1}{2}\int^1_0 \|\tilde{b}(y,t;\beta) \|^2 \dif t + \E\left[F_\theta\left(W + \int^\bullet_0 \tilde{b}(y,s;\beta) \dif s\right) \right]\Bigg\}.
\end{align*}
One key difference from the DLGM set-up is worth mentioning: here, the only degree of freedom we need is an additive drift that affects the mean, whereas in the DLGM case we optimize over both the mean and the covariance matrix in Eq.~\eqref{eq:FE}.

\section{{Automatic differentiation in Wiener space}}
%\subsection{Overview of different techniques}
We are now faced with the problem of computing the gradients of the variational free energy
\begin{align}\label{eq:diffusion_FE}
	\sF_{\theta,\beta}(y) &\deq \frac{1}{2}\int^1_0 \|\tilde{b}(y,t;\beta)\|^2 \dif t  + \E_{\bdmu}\left[F_\theta\left(W + \int^\bullet_0 \tilde{b}(y,s;\beta) \dif s\right) \right]
\end{align}
with respect to $\theta$ and $\beta$. The gradients of the first (KL-divegence) term on the right-hand side \eqref{eq:diffusion_FE} can be computed straightforwardly using automatic differentiation, so we turn to the second term. To that end, let us define, for each $\theta$ and $\beta$, the It\^o process $X^{\theta,\beta}$ by
\begin{align}\label{eq:Xthetabeta}
	 X^{\theta,\beta}_t \deq X_0 + \int^t_0 b(X^{\theta,\beta}_s,s; \theta)\dif s + \int^t_0 \tilde{b}(y,s;\beta) \dif s + \int^t_0 \sigma(X^{\theta,\beta}_s,s;\theta) \dif W_s, \,\, t \in [0,1]
\end{align}
which is simply the result of adding the deterministic drift term $\tilde{b}(y,t;\beta) \dif t$ to the neural SDE \eqref{eq:neural_SDE}. Then the second term on the right-hand side of \eqref{eq:diffusion_FE} is equal to $-\E[\log p(y|X^{\theta,\beta}_1)]$, and we need a procedure for computing the gradients of this term w.r.t.\ $\theta$ and $\beta$. Provided we have a way of differentiating the It\^o process \eqref{eq:Xthetabeta} w.r.t.\ $\theta$ and $\beta$, the desired gradients are given by
\begin{align*}
	\frac{\partial}{\partial \beta} \E[\log p(y|X^{\theta,\beta}_1)] &= \E\left[\frac{\partial}{\partial x}\log p(y|X^{\theta,\beta}_1) \frac{\partial X^{\theta,\beta}_1}{\partial \beta}\right] , \\
	\frac{\partial}{\partial \theta} \E[\log p(y|X^{\theta,\beta}_1)] &=  \E\left[\frac{\partial}{\partial x} \log p(y|X^{\theta,\beta}_1) \frac{\partial X^{\theta,\beta}_1}{\partial \theta}\right].
\end{align*}
(Here, we are assuming that the function $x \mapsto \log p(y|x)$ is sufficiently well-behaved to permit interchange of differentiation and integration.)

In the remainder of this section, we first compare the problem of gradient computation in neural SDEs to its deterministic counterpart in neural ODEs and then describe two possible approaches. 

\subsection{Gradient computation in neural SDEs}
\label{ssec:SDEgrad}

Consider the following problem: We have a $d$-dimensional It\^o process
\begin{align}\label{eq:generic_SDE}
	\dif X^\alpha_t = b(X^\alpha_t,t;\alpha)\dif t + \sigma(X^\alpha_t,t;\alpha)\dif W_t, 
\end{align}
where $\alpha$ is a $p$-dimensional parameter, and a function $f : \Reals^d \to \Reals$ be given. We wish to compute the gradient of the expectation $J(\alpha) \deq \E[f(X^\alpha_1)]$ w.r.t.\ $\alpha$. In the context of variational inference for neural SDEs, we have $\alpha = (\theta,\beta)$ and $f(\cdot) = -\log p(y|\cdot)$. This problem is known as \textit{sensitivity analysis} \citep{gobet2005malliavin}. We are allowed to use automatic differentiation and black-box SDE or ODE solvers as building blocks. 

In the deterministic case, i.e., when $\sigma(\cdot) \equiv 0$, Eq.~\eqref{eq:generic_SDE} is an instance of a neural ODE \citep{chen18neuralODE}, and the computation of $\frac{\partial J(\alpha)}{\partial \alpha}$ can be carried out efficiently using any black-box ODE solver. The key idea, based on the so-called \textit{adjoint sensitivity method} (see, e.g., \citet{kokotovic67adjoints} and references therein), is to augment the original ODE that runs \textit{forward} in time with a certain second ODE that runs \textit{backward} in time. This allows one to efficiently backpropagate gradients through \textit{any} black-box ODE solver. 

Unfortunately, there is no straightforward way to port this construction to SDEs. In very broad strokes, this can be explained as follows: While the analogue of the method of adjoints is available for SDEs (see, e.g., \citet[Chap.~3]{yongzhou_HJB}), the adjoint equation is an SDE that has to be solved \textit{backward} in time with a terminal condition that depends on the \textit{entire} Wiener path $W = \{W_t\}_{t \in [0,1]}$, but the solution at each time $t \in [0,1]$ must still be measurable only w.r.t.\ the ``past'' $\{W_s\}_{s \in [0,t]}$. The augmented system consisting of the original forward SDE and the adjoint backward SDE is an instance of a \textit{forward-backward SDE}, or FBSDE for short \citep[Chap.~7]{yongzhou_HJB}. To the best of our knowledge, there are no efficient  black-box schemes for solving FBSDEs with computation requirements comparable to standard SDE solvers; this stems from the fact that any procedure for solving FBSDEs must rely on a routine for solving a certain class of semilinear parabolic PDEs \citep{milstein2006FBSDE}, which will incur considerable computational costs in high-dimensional settings. (There are, however, promising first steps in this direction by \citet{han2018BSDE,han2019deepBSDE} based on deep neural nets.)

This unfortunate complication means that we have to forgo the use of adjoint-based methods for SDEs and instead develop gradient computation procedures by other means. We describe two possible approaches in the remainder of this section. As stated earlier, we assume that the following building blocks are available:
\begin{itemize}
	\item \textit{automatic differentiation} (or AD, \citet{griewank2008AD,baydin2018autodiff}): Given a straight-line program (directed acyclic computation graph) for computing the value of a differentiable function $f : \Reals^m \to \Reals^n$ at any point $v \in \Reals^m$, AD is a systematic procedure that generates a program for computing the Jacobian of $f$ at $v$.  The time complexity of the resulting program for evaluating the Jacobian scales as $c \cdot m\sT(f)$ using forward-mode AD or $c \cdot n \sT(f)$ using reverse-mode AD, where $\sT(f)$ is the time complexity of the program for computing $f$ and $c$ is a small absolute constant. In particular, one should use reverse-mode AD when $n \ll m$; however, the space complexity of the program generated by reverse-mode AD may be rather high.
	\item a \textit{black-box SDE solver}: For a given initial condition $z \in \Reals^m$, initial and final times $0 \le t_0 < t_1 \le 1$, drift $f : \Reals^m \times [0,1] \to \Reals^m$, and diffusion coefficient matrix $g : \Reals^m \times [0,1] \to \Reals^{m \times d}$, we will denote by ${\sf SDE.Solve}(z,t_0,t_1,f,g)$ the output of any black-box procedure that computes or approximates the solution at time $t_1$ of the It\^o SDE $\dif Z_t = f(Z_t,t)\dif t + g(Z_t,t)\dif W_t$, $t \in [t_0,t_1]$, with $Z_{t_0} = z$. 
\end{itemize} 
We assume, moreover, that one can pass straight-line programs for computing $f$ and $g$ as arguments to  {\sf SDE.Solve}.

\subsection{Solve-then-differentiate: the Euler backprop}
\label{ssec:Euler_backprop}

The most straightforward approach is to derive custom backpropagation equations for a specific SDE solver, e.g., the Euler method. We first generate an Euler approximation of the diffusion process \eqref{eq:Xthetabeta} and then estimate the gradients of $-\E[\log p(y|X^{\theta,\beta}_1)]$ w.r.t.\ $\theta$ and $\beta$ by backpropagation through the computation graph of the Euler recursion:
\begin{itemize}
	\item the forward pass --- given a time mesh $0 = t_0 < t_1 < \ldots < t_N = 1$, we sample $Z_1,\ldots,Z_N \stackrel{{\rm i.i.d.}}{\sim} \gamma_d$ and the desired initialization $\wh{X}^{\theta,\beta}_{t_0}$, and generate the updates
\begin{align*}
	\wh{X}^{\theta,\beta}_{t_{i+1}} &= \wh{X}^{\theta,\beta}_{t_i} + h_{i+1}\left(b(\wh{X}^{\theta,\beta}_{t_i},t_i; \theta) + \tilde{b}(y, t_i; \beta)\right) \nonumber\\
	& \qquad + \sqrt{h_{i+1}} \sigma(\wh{X}^{\theta,\beta}_{t_i},t_i; \theta) Z_{i+1}
\end{align*}
for $0 \le i < N$, where $h_{i+1} \deq t_{i+1}-t_i$;
\item the backward pass --- compute the gradients of $-\log p(y|\wh{X}^{\theta,\beta}_{t_N})$ w.r.t.\ $\theta$ and $\beta$ using reverse-mode AD.
\end{itemize}
The overall time complexity is roughly $O\big(N\cdot({\sT}(b)+{\sT}(\tilde{b})+{\sT}(\sigma))\big)$,  on the order of the time complexity of the forward pass --- exactly as one would expect for reverse-mode AD. This approach, which approximates the continuous-time neural SDE by a discrete DLGM of the form \eqref{eq:DLGM}, has previously been used in mathematical finance \citep{SmokingAdjoints} and, more recently, in the context of variational inference for SDEs \citep{ryder2018SDE_VI}. The forward and the backward passes have the same structure as in the original work of \citet{rezende2014stochbackprop}, and Monte-Carlo estimates of the gradient can be obtained by averaging over multiple independent realizations.

\subsection{Differentiate-then-solve: the pathwise differentiation method}
\label{ssec:pathwise}

In some settings, it may be desirable to avoid explicit discretization of the neural SDE and work with a black-box SDE solver instead. This approach amounts to differentiating through the forward pass of the solver. The computation of the pathwise derivatives of $X^{\theta,\beta}$ with respect to $\theta$ or $\beta$ can be accomplished by solving another SDE, as a consequence of the theory of stochastic flows \citep{kunita84flows}.

\begin{theorem}\label{thm:pathwise} Assume the following:
	\begin{enumerate}
		\item The drift $b(x,t;\theta)$ and the diffusion matrix $\sigma(x,t;\theta)$ are Lipschitz-continuous with Lipschitz-continuous Jacobians in $x$ and $\theta$, uniformly in $t \in [0,1]$.
		\item The drift $\tilde{b}(y,t;\beta)$ is Lipschitz-continuous with Lipschitz-continuous Jacobian in $\beta$, uniformly in $t \in [0,1]$.
	\end{enumerate}
	Then the pathwise derivatives of $X = X^{\theta,\beta}$ in $\theta$ and $\beta$ are given by the following It\^o processes:
	\begin{align}
		\frac{\partial X_t}{\partial \beta^i} &= \int^t_0 \Bigg( \frac{\partial b_s}{\partial x} \frac{\partial X_s}{\partial \beta^i} + \frac{\partial \tilde{b}_s}{\partial \beta^i}\Bigg) \dif s   + \sum^d_{\ell=1}\int^t_0 \frac{\partial \sigma_{s,\ell}}{\partial x} \frac{\partial X_s}{\partial \beta^i} \dif W^\ell_s \label{eq:Jbeta}
	\end{align}
	and
	\begin{align}
		\frac{\partial X_t}{\partial \theta^j} &= \int^t_0 \Bigg( \frac{\partial b_s}{\partial \theta^j} + \frac{\partial b_s}{\partial x} \frac{\partial X_s}{\partial \theta^j} \Bigg) \dif s   + \sum^d_{\ell=1}\int^t_0 \Bigg(\frac{\partial\sigma_{s,\ell}}{\partial \theta^j} +  \frac{\partial \sigma_{s,\ell}}{\partial x} \frac{\partial X_s}{\partial \theta^j} \Bigg) \dif W^\ell_s \label{eq:Jtheta}
	\end{align}
	where, e.g., $b_s$ is shorthand for $b(X^{\theta,\beta}_s,s;\theta)$, $\sigma_{s,\ell}$ denotes the $\ell$th column of $\sigma(X^{\theta,\beta}_s,s;\theta)$, and $W^1,\ldots,W^d$ are the independent scalar coordinates of the $d$-dimensional Wiener process $W$.
\end{theorem}

\begin{proof}
	We will use the results of \citet{kunita84flows} on differentiability of the solutions of It\^o SDEs w.r.t.\ initial conditions. Consider an $m$-dimensional It\^o process of the form
		\begin{align}\label{eq:Zxi}
			Z_t(\xi) = \xi + \int^t_0 f_0(Z_s(\xi),s) \dif s + \sum^m_{i=1}\int^t_0 f_i(Z_s(\xi),s) \dif W^i_s, \qquad t\in [0,1]
		\end{align}
		where $W^1,\ldots,W^m$ are independent standard scalar Brownian motions, with the following assumptions:
		\begin{enumerate}
			\item The initial condition $Z_0(\xi) = \xi \in \Reals^m$.
			\item The vector fields $f_j: \Reals^m \times [0,1] \to \Reals^m$, $j \in \{0,\ldots,m\}$, are globally Lipschitz-continuous and have Lipschitz-continuous gradients, uniformlly in $t \in [0,1]$.
		\end{enumerate}
		Then \citep[Chap.~2, Thm.~3.1]{kunita84flows} the pathwise derivatives of $Z$ w.r.t.\ the initial condition $\xi$ exist and are given by the It\^o processes
		\begin{align}\label{eq:Jxi}
			\frac{\partial Z_t(\xi)}{\partial \xi^i} = e_i + \int^t_0 \frac{\partial}{\partial z} f_0(Z_s(\xi),s) \frac{\partial Z_s(\xi)}{\partial \xi^i} \dif s + \sum^m_{j=1}\int^t_0 \frac{\partial}{\partial z} \sigma_j(Z_s(\xi),s) \frac{\partial Z_s(\xi)}{\partial \xi^i} \dif W^j_s,
		\end{align}
		where $e_1,\ldots,e_m$ are the standard basis vectors in $\Reals^m$. Now suppose that $\beta \in \Reals^k$, $\theta \in \Reals^n$, let $m \deq d + k + n$, and define the vector fields $f_0(z,t)$ for $z = (x^\top, \beta^\top, \theta^\top)^\top \in \Reals^m$ and $t \in [0,1]$ by
	\begin{align}\label{eq:fields}
		f_0(z,t) \deq \left( \begin{matrix}
		b(x,t;\theta) + \tilde{b}(y,t;\beta) \\
		0
	\end{matrix}\right) \text{ and } f_j(z,t) \deq  \left( \begin{matrix}
	\sigma_j(x,t;\theta) \cdot \1_{j \in \{1,\ldots,d\}} \\
	0
	\end{matrix}\right),
	\end{align}
	where $\sigma_j(x,t;\theta)$ is the $j$th column of $\sigma(x,t;\theta)$. These vector fields satisfy the above Lipschitz continuity condition by hypothesis. Consider the It\^o process \eqref{eq:Zxi} with the initial condition $\xi = (0^\top, \beta^\top, \theta^\top)^\top$. Then evidently
	\begin{align*}
		Z_t(\xi) = \left( \begin{matrix} X^{\theta,\beta}_t \\
		\beta \\
		\theta \end{matrix}\right), \qquad t \in [0,1]
	\end{align*}
	and Eqs.~\eqref{eq:Jbeta} and \eqref{eq:Jtheta} follow from \eqref{eq:Jxi}, \eqref{eq:fields}, and the chain rule of multivariable calculus.
	\end{proof}

From the above, it follows that we can compute the gradients of the variational free energy \eqref{eq:diffusion_FE} using any black-box SDE solver. To that end, we first use AD to generate the programs for computing the Jacobians of $\log p(y|\cdot)$, $b$, $\tilde{b}$, and $\sigma$, which, together with $b$, $\tilde{b}$, and $\sigma$, can then be supplied to the SDE solver. We can then obtain both $X^{\beta,\theta}$ and the pathwise derivatives $\frac{\partial X^{\beta,\theta}}{\partial \beta}$ and $\frac{\partial X^{\beta,\theta}}{\partial \theta}$ by a call to ${\sf SDE.Solve}$ with $f$ and $g$ consisting of $b(\cdot;\theta)$, $\sigma(\cdot;\theta)$, $\tilde{b}(\cdot;\beta)$ and their Jacobians w.r.t.\ $x$, $\beta$, and $\theta$.

The time complexity is determined by the internal workings of ${\sf SDE.Solve}$ and by the time complexity of computing the Jacobians that enter the drift and the diffusion coefficients in \eqref{eq:Jbeta} and \eqref{eq:Jtheta}. Suppose that $\theta$ takes values in $\Reals^n$ and $\beta$ takes values in $\Reals^k$. Then:
\begin{itemize}
	\item for \eqref{eq:Jbeta}, we need the Jacobian of $b(x,t;\theta)$ w.r.t.\ $x$, the Jacobian of $\tilde{b}(y;\beta)$ w.r.t.\ $\beta$, and the Jacobian of $\sigma(x,t;\theta)$ w.r.t.\ $x$. If we use AD to generate the programs for computing the Jacobians, the total time complexity per iteration of the SDE solver will be
	$$
	O\Big(k\cdot\big(d\cdot {\sT}(b) + \min(d,k) \cdot {\sT}(\tilde{b}) + d \cdot {\sT}(\sigma)\big)\Big),
	$$
	using forward-mode or reverse-mode AD as needed.
	\item for \eqref{eq:Jtheta}, we need the Jacobian of $b(x,t;\theta)$ and $\sigma(x,t;\theta)$ w.r.t. $\theta$ and $x$. The total per-iteration time complexity of generating the Jacobians using AD will be
	$$
	O\Big(n(\min(d,n)+d)\cdot \big({\sT}(b) + {\sT}(\sigma)\big)\Big).
	$$
	Typically, the dimension $n$ of the latent parameter $\theta$ will be on the order of $d^2$ for a fully connected neural net. 
\end{itemize}
The solve-then-differentiate approach of Section~\ref{ssec:Euler_backprop} will generally scale better to high-dimensional problems than the  black-box pathwise approach. On the other hand, the pathwise approach is more flexible since it works with a generic SDE solver, so the overall time complexity may be reduced by using an adaptive SDE solver \citep{ilie2015sdesolve}. In addition, since the pathwise approach amounts to differentiating through the forward operation of the solver, it may incur smaller storage overhead than the Euler backprop method.

 As before, the gradients of the free energy w.r.t.\ $\beta$ and $\theta$ can be estimated using Monte Carlo methods, by averaging multiple independent runs of the SDE solver.
 % It is worth mentioning that, in contrast to the neural ODE framework of \citet{chen18neuralODE}, where one has to augment the original ODE with the so-called adjoint ODE that has to be solved \textit{backwards} in time to obtain the gradients w.r.t.\ the model parameters, here we run our SDE solver forwards in time.
 
 \begin{figure}[htb]
 \centering
 \begin{minipage}{1\linewidth}
 \centering
 \includegraphics[width=\linewidth]{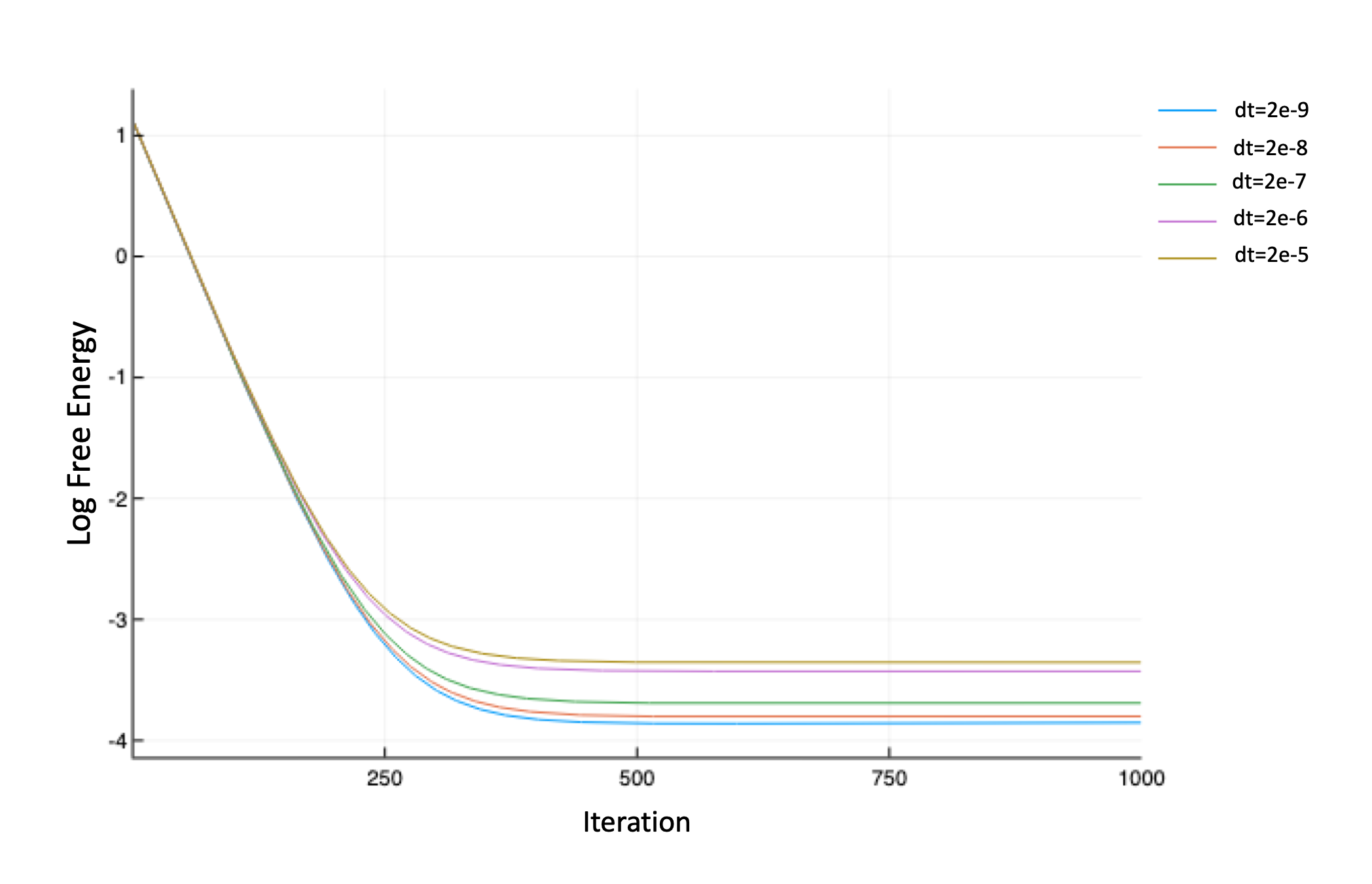}
 \caption{\small Log free energy for different discretization mesh sizes (sample size $n=1000$).}
 \label{fig:discretization}
 \end{minipage}%
 \hfill
 \begin{minipage}{1\linewidth}
 \centering
 \includegraphics[width=\linewidth]{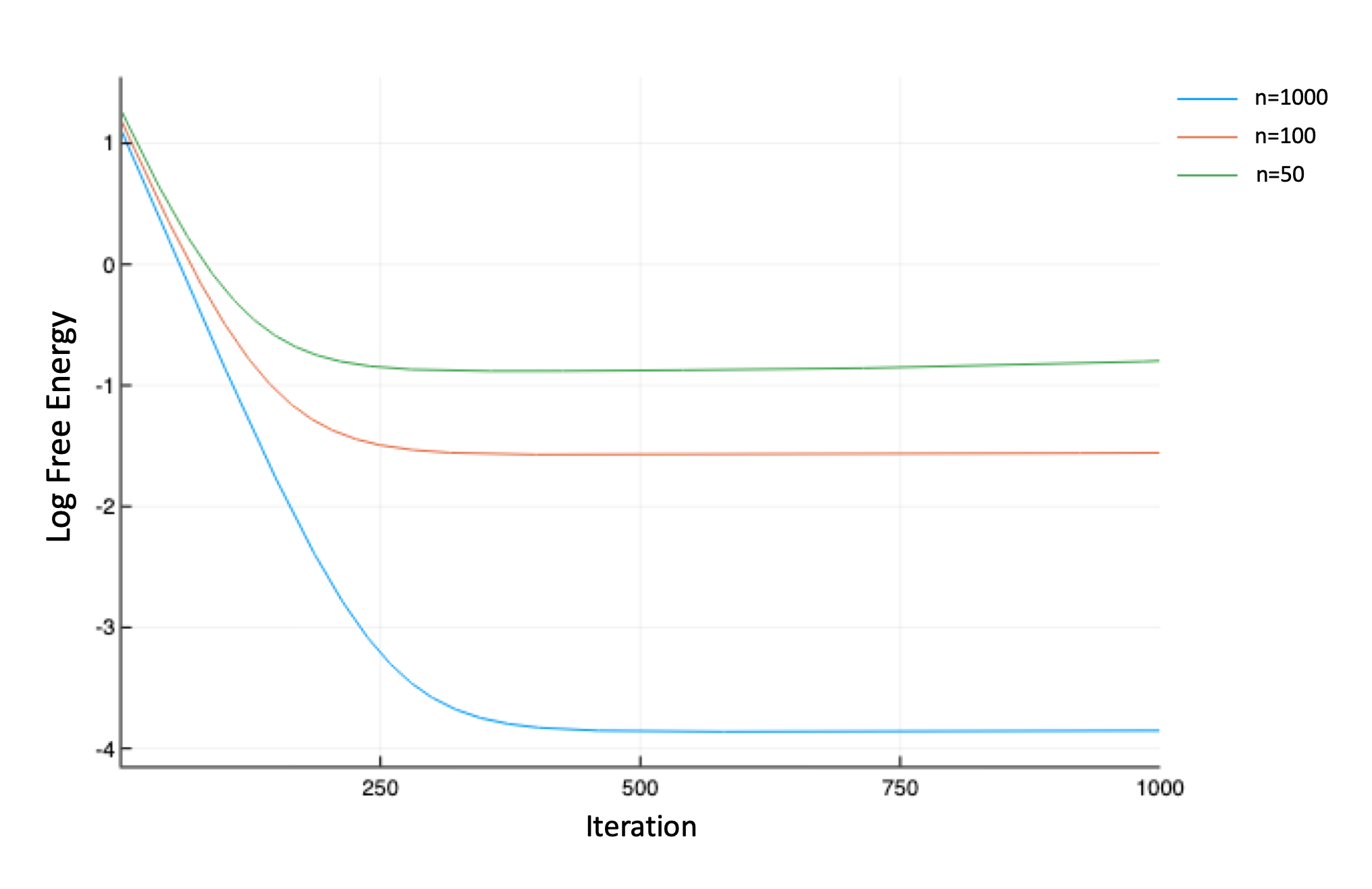}
 \caption{\small Log free energy for different sample sizes (discretization mesh size $h=\frac{1}{32}$).}
 \label{fig:sample}
 \end{minipage}
 \end{figure}

\section{Experimental results}

We evaluated the performance of the forward-differentiation pathwise method of Sec.~\ref{ssec:pathwise} using gradient descent on synthetic data.  The code for all experiments was written in Julia using the DiffEqFlux library \citep{DiffEqFlux} and executed on a CPU. The synthetic data were generated via a numerical SDE solution of the diffusion process
\begin{align}\label{eq:groundtruth}
\dif X_t = {\sf sigmoid}(A X_t) \dif t + \dif W_t, \qquad t\in [0,1]
\end{align}
where ${\sf sigmoid}(\cdot)$ is the coordinatewise sigmoid softplus activation function.   The observations $Y_1,\ldots,Y_n \in \Reals^d$ were generated by sampling $n$ independent  copies of this process and then addding small independent Gaussian perturbations to each copy of $X_1$. That is, we use \eqref{eq:groundtruth} parametrized by $A \in \Reals^{d\times d}$ as the neural SDE model, and take $p(y|x)$ to be Gaussian with mean $x$ and covariance matrix $I_d$. The auxiliary parameters were chosen from the class of constant drift terms independent of the process $X_t$, corresponding to  mean-field approximations of the form
\begin{align*}%\label{eq:varMF}
\dif X^{A,\beta}_t =  \big({\sf sigmoid}(AX^{A,\beta}_t) + \beta\big) \dif t + \dif W_t.
\end{align*}
The ground-truth parameter $A \in \Reals^{d \times d}$ was randomly chosen with elements drawn i.i.d.\ from $\cN(0,1)$; all other parameters were likewise randomly initialized.

For appropriate choices of parameters, performing vanilla gradient descent with constant step size leads to a decrease of the free energy.   We examined the effect on performance of different discretization step sizes, as well as different sample sizes. The results are shown in Figures~\ref{fig:discretization} and \ref{fig:sample} for a $100$-dimensional ground-truth parameter $A \in \Reals^{10 \times 10}$ and the variational approximation drift $\beta \in \Reals^{10}$.
 
Interestingly, the improvements from discretization meshes finer than $h = 1/32$ are incremental, suggesting that, at least in a simple case such as this, models with ``infinitely many layers'' may  not offer significant practical advantage over models with ``finitely many" layers, such as DLGMs, where the number of layers approaches the dimensionality of the problem.  On the other hand, while the method exhibits some robustness when data are not abundant ($n\approx d$), doing not much worse than when they are ($n \gg d$), it does increasingly worse---reflected in log-likelihood attained by optimized parameters---and ultimately fails in a low-data regime, as $n$ approaches $\sqrt{d}$.

\section{Conclusion and future directions}

We have presented an analysis of neural SDEs, which can be viewed as a continuous-time limit of the DLGMs of \citet{rezende2014stochbackprop} or as a stochastic version of the neural ODEs of \cite{chen18neuralODE}.  In particular, this is a best-of-both-worlds perspective that enables us to both reason about the compositional expressive power of DLGMs and the inference process in the space of measures. In addition, it allows us to draw on a variety of standard scientific computing methods developed for continuous-time stochastic processes.

We have shown that optimization for neural SDEs is far from a straightforward analogue of the ODE case, with the issue of time-adaptedness of paths complicating the use of the adjoint sensitivity method for gradient computation.  The discretize-then-differentiate approach was shown to exactly recover the stochastic backpropagation method of \citet{rezende2014stochbackprop} when using a simple Euler discretization, with the form of the mean-field variational approximation closely preserved; and the differentiate-then-discretize pathwise approach was demonstrated to be a single-pass method that leverages a black-box SDE solver, with little storage overhead but high computational complexity and thus limited scalability due to the use of forward-mode differentiation.

One interesting direction to examine going forward is more sophisticated discretization schemes that readily enable the use of numerical tools developed for continuous-time processes (e.g., ODE solvers), where the continuous-time process is viewed as a deterministic function of random increments, perhaps arbitrarily small as the discretization becomes increasingly fine. Another promising direction is to investigate the connection between neural SDEs and probabilistic ODE solvers that return a posterior estimate rather than a deterministic approximate solution \citep{conrad2017probODE,schober2019probODE}.

\subsection*{Acknowledgments} The authors would like to thank Matus Telgarsky for many enlightening discussions, and Chris Rackauckas, Markus Heinonen, and Mauricio \'Alvarez for their comments and constructive suggestions on the first version of this work. This work  was supported in part by the NSF CAREER award CCF-1254041, in part by the Center for Science of Information (CSoI), an NSF Science and Technology Center, under grant agreement CCF-0939370, in part by the Center for Advanced Electronics through Machine Learning (CAEML) I/UCRC award no.~CNS-16-24811, and in part by the Office of Naval Research under grant no.~N00014-12-1-0998.

%\newpage

\bibliography{neural_SDE_v2.bbl}

\end{document}